\documentclass{article}

\usepackage[preprint]{neurips_2024}

\usepackage[utf8]{inputenc}
\usepackage[T1]{fontenc}
\usepackage{amsmath,amssymb,amsthm,mathtools}
\usepackage{hyperref}
\usepackage{url}
\usepackage{booktabs}
\usepackage{enumitem}
\usepackage{graphicx}
\usepackage{natbib}
\usepackage{xcolor}
\usepackage{algorithm}
\usepackage{algorithmic}
\usepackage{float}

\theoremstyle{plain}
\newtheorem{assumption}{Assumption}
\newtheorem{proposition}{Proposition}
\newtheorem{theorem}{Theorem}
\newtheorem{corollary}{Corollary}

\newtheorem{definition}{Definition}
\theoremstyle{remark}

\newcommand{\E}{\mathbb{E}}
\newcommand{\Pbb}{\mathbb{P}}

\newcommand{\cL}{\mathcal{L}}

\newcommand{\cH}{\mathcal{H}}
\newcommand{\cN}{\mathcal{N}}
\newcommand{\Var}{\mathrm{Var}}
\newcommand{\Cov}{\mathrm{Cov}}

\newcommand{\Rmix}{R_{\mathrm{mix}}}
\newcommand{\Rweight}{R_{\mathrm{weight}}}
\newcommand{\Beff}{B_{\mathrm{eff}}}
\newcommand{\hatmu}{\widehat{\mu}}
\newcommand{\hattau}{\widehat{\tau}}
\newcommand{\hatm}{\widehat{m}}
\newcommand{\hatI}{\widehat{I}}
\newcommand{\hatY}{\widehat{Y}}
\newcommand{\barY}{\bar{Y}}
\newcommand{\barI}{\bar{I}}

\title{How Wrong Can Your Counterfactual Be? Quantifying Confounding Bias for Continuous Treatments without a Control Group}

\author{
  Yu Wang\thanks{Department of Economics, Cornell University. Work initiated while at Cornell University. Correspondence to \texttt{yw2322@cornell.edu}.}
  \and
  Xiangchen Liu\thanks{Department of Family and Consumer Sciences, California State University, Long Beach. \texttt{xiangchen.liu@csulb.edu}.}
  \and
  Siguang Li\thanks{Society Hub, Hong Kong University of Science and Technology (Guangzhou). \texttt{siguangli@hkust-gz.edu.cn}.}
}

\begin{document}
\maketitle

%% ==================================================================
%% ABSTRACT
%% ==================================================================
\begin{abstract}
Regulatory stress testing poses a causal question: how would portfolio credit losses change if the macroeconomy followed an adverse counterfactual path? Yet standard practice remains predictive and is therefore vulnerable to omitted-variable bias. We propose a partial identification framework for causal stress testing in panel data with a continuous common treatment and no control group. By assuming that the unobserved confounder affects outcome and macro variables additively, we derive a closed-form confounding envelope parameterized by two interpretable sensitivity parameters. We further analyze two practical estimators—recursive rollout and direct multi-horizon prediction—derive non-asymptotic error bounds, and characterize when recursive compounding makes direct estimation preferable. For inference, we combine the identification envelope with importance-weighted conformal prediction, yielding finite-sample intervals that separate estimation uncertainty from identification uncertainty under covariate shift. In semi-synthetic experiments built from real U.S. unemployment paths, standard high-accuracy predictive models remain causally biased and substantially under-cover, whereas the proposed framework achieves near-nominal coverage across stress horizons.
\end{abstract}

%% ==================================================================
%% 1. INTRODUCTION
%% ==================================================================
\section{Introduction}
\label{sec:intro}

Regulatory stress testing\footnote{Under the Dodd--Frank Act, the Federal Reserve's Comprehensive Capital Analysis and Review requires banks to project credit losses under hypothetical macroeconomic scenarios \citep{federalreserve2024ccar}.}, a central modeling task at large financial institutions, asks: how would a bank's credit losses respond if the macroeconomy---say, unemployment and house prices---followed a prescribed adverse path? Answering this question accurately is fundamentally a causal problem---it is a statement about what \emph{would} happen under a counterfactual intervention on the macro variables. However, the methods used in practice are predictive: banks fit regression or machine learning models on historical data and extrapolate under the stress scenario \citep{petropoulos2022ml}. 

Fundamentally, any machine learning model that treats the macro variables as exogenous inputs suffers from an omitted variable problem---unobserved factors that drive both the macro path and credit outcomes are left out of the model, systematically biasing counterfactual predictions \citep{BernankeGertlerGilchrist1999,AdrianBoyarchenkoGiannone2019,ramey2016shocks,bica2020timeseriesdeconfounder}. Surprisingly, this tension has received little attention, with \citet{gao2017causalstress} and \citet{mauri2023causalAI} the notable exceptions, both of which propose learning a causal network from data and using it to simulate stress scenarios. While a useful first step, such approaches do not provide a formalized framework for causal stress testing.

Two features of stress testing render classical causal inference methods---from synthetic controls \citep{abadie2010synthetic} to their ML counterparts like matrix completion \citep{athey2021matrix}---inapplicable: the absence of a control group and the continuous nature of the treatment. The hypothetical macro scenario affects all units in the economy simultaneously, and its value varies continuously rather than taking a binary form. Recently, \citet{cerqua2023mlcm} propose an identification and estimation framework for counterfactual forecasting without a control group, but the method is limited to binary treatment. Together, these obstacles make point identification infeasible without implausible assumptions.

In this paper, we propose a partial identification and estimation framework for causal stress testing. Partial identification, pioneered by \citet{manski1990bounds,manski2003partial} and recently extended to the ML setting \citep{kallus2019confounding,zhang2022partial,frauen2024neuralcsa}, is designed for settings where the data and assumptions are not strong enough to achieve point identification, but are strong enough to achieve set identification: a credible range of values consistent with what we know. 

Our framework operates on individual-level panel data---outcomes observed for many accounts over many time periods\footnote{For example, a credit card portfolio where each account's monthly balance, payment, and delinquency status are tracked over time, while the macroeconomy (unemployment, interest rates) moves in the background.}---where the goal is to estimate how credit risk changes under a hypothetical macro scenario. The identification difficulty stems from an omitted variable problem: an unobserved variable simultaneously affects both individual-level credit risk and the macro variable. The key assumption is that the unobserved confounder enters both the outcome and the macro variable additively, which is aligned with the literature on latent financial conditions as common drivers of macroeconomic and credit dynamics \citep{BernankeGertlerGilchrist1999,AdrianBoyarchenkoGiannone2019}.

We first derive the closed-form bound explicitly in terms of two interpretable sensitivity parameters \citep{carnegie2016sensitivity,cinelli2020sensemakr}, which govern the strength of confounding in the macro and outcome equations. These parameters can be anchored empirically using auxiliary data such as the Chicago Fed National Financial Conditions Index, making the bound fully computable from observable quantities.

For estimation, we consider two complementary strategies: the recursive rollout estimator used in regulatory practice \citep{BoardOfGovernors2023,covas2015topdown}, which feeds each period's projected outcome into the next period's prediction, and a direct multi-horizon estimator that predicts each horizon's outcome in a single step. We establish non-asymptotic error bounds for both. The recursive error compounds through a horizon-dependent amplification factor, which we use to derive a practical selection rule: recursive estimation is reliable when the factor is small, and direct estimation is preferable otherwise.

For inference, we adapt importance-weighted conformal prediction \citep{tibshirani2019covariate,barber2023conformal} to the stress-testing setting. Conformal prediction provides distribution-free, finite-sample coverage guarantees that hold regardless of the underlying model class, and the importance-weighted version corrects for the covariate shift between historical calibration data and the stress scenario. Combining the identification bound with the calibration band gives a confidence interval for the causal stress-period mean that decomposes into two independently interpretable layers---estimation uncertainty and identification uncertainty---each with its own formal guarantee.

Real data do not provide ground truth for counterfactual outcomes---we cannot observe what credit losses would have been under a macro scenario that did not materialize. We therefore validate the framework on semi-synthetic experiments that combine real U.S. unemployment paths (FRED UNRATE, 1990--2024) with synthetic credit outcomes generated from a known data-generating process (DGP). In the experiments, we consider both a linear DGP and a nonlinear DGP with a quadratic macro response and a threshold amplifier for more realistic dynamics. The framework is agnostic to the choice of learner: we use OLS for the linear DGP and a two-layer MLP for the nonlinear DGP, but any base predictor can be plugged into the conformal calibration and confounding envelope. 

Our experiments first show that a predictive model with high backtesting accuracy can be systematically far off from the true causal mean under stress, while its own reported uncertainty is orders of magnitude too small to reflect this error. We then show that only our full framework---weighted conformal bands with the confounding envelope---achieves nominal coverage across all stress horizons, outperforming OLS regression with standard errors, bootstrap, ensemble conformal prediction \citep{xu2021conformal}, and weighted conformal bands without the envelope. Finally, we also provide an example of calibrating the confounding factors strength from observed data, using it in computing stress testing.

We consider contributions of our paper are three-fold. To the stress testing literature, we show that the predictive--causal gap in regulatory stress tests is structural and invisible to standard validation, and provide a complete framework---identification, estimation, and inference---that delivers valid coverage of the causal estimand with a concrete calibration protocol using observable data. To the causal inference/ML literature, we formalize a setting with no control group and continuous treatment that falls outside standard frameworks, and show that stationarity enables temporal pooling to substitute for cross-sectional variation, yielding a closed-form partial identification bound that decomposes total uncertainty into orthogonal estimation and identification layers. Third, we provide a complete inference procedure (weighted conformal prediction + confounding envelope) that delivers valid coverage of the causal estimand, and demonstrate an empirical calibration protocol using observable financial conditions data.

%% ==================================================================
%% 2. RELATED WORK
%% ==================================================================

%% ==================================================================
%% 3. SETUP AND ESTIMANDS
%% ==================================================================
\section{Setup and Estimands}
\label{sec:setup}

Consider panel data with units $i = 1, \ldots, N$ observed over periods $t = 1, \ldots, T$. Let $Y_{i,t}$ denote the outcome (e.g., credit loss), $X_{i,t}$ observed covariates, and $A_t$ a macro variable common across units (e.g., unemployment rate). A stress scenario specifies a hypothetical path $a^S = (a^S_1, \ldots, a^S_H)$ for $A$ over a forecast horizon $h = 1, \ldots, H$, optionally contrasted against a baseline path $a^B$.

\begin{definition}[Predictive and causal estimands]
\label{def:estimands}
For a macro path $a = (a_1,\ldots,a_H)$, the \emph{predictive mean} is
$\mu_h(a) := \E[Y_{i,h} \mid A_{1:h} = a_{1:h}]$
and the \emph{causal (interventional) mean} is
$\mu^*_h(a) := \E[Y_{i,h}(a_{1:h})]$.
\end{definition}

The two estimands differ in what drives the macro path. Under intervention, $A$ is set externally; under observation, conditioning on $A_{1:h} = a$ also selects on whatever unobserved factors drove $A$ to that value. When unemployment rises to 10\%, a predictive model attributes all the resulting losses to unemployment, but part of those losses reflect the deteriorating financial conditions that caused unemployment to rise. The gap $\mu_h(a) - \mu^*_h(a)$ is precisely this confounding bias. The stress test asks for $\mu^*_h(a^S)$; industry practice returns $\mu_h(a^S)$, and the gap is precisely the confounding bias.

Throughout, we assume unit $i$'s history can be compressed into a finite-dimensional state---a minimal condition for a supervised model trained on historical data to generalize to a hypothetical macro path. Let $\mathcal{H}_{i,t} := \{Y_{i,s}, X_{i,s}, A_s\}_{s \le t}$ denote the history of unit $i$ up to time $t$, and $I_{i,t} = g(\mathcal{H}_{i,t})$ a finite-dimensional summary. Formally, we have the following two assumptions:

\begin{assumption}[Markov state]
\label{ass:state} $\E[Y_{i,t+1} \mid I_{i,t}, A_{t+1}] = \E[Y_{i,t+1} \mid \mathcal{H}_{i,t}, A_{t+1}].$
\end{assumption}

\begin{assumption}[Stationarity]
\label{ass:stationary}
The conditional law $\cL(Y_{i,t+1} \mid I_{i,t}, A_{t+1})$ does not depend on $t$.
\end{assumption}

Under Assumptions~\ref{ass:state}--\ref{ass:stationary}, the predictive mean $\mu_h(a)$ is estimable from pre-stress data by iterated conditional expectations---a standard prediction problem requiring no causal assumption. Whether it equals $\mu^*_h(a)$ is the question the rest of the paper addresses.

%% ==================================================================
%% 4. IDENTIFICATION
%% ==================================================================
\section{Causal Set Identification under Structural Confounding}
\label{sec:id}

The core economic concern is that unobserved factors---such as financial conditions, credit market tightness, or investor sentiment---simultaneously drive both the macro variable and the outcome. When financial conditions deteriorate, unemployment rises and credit losses increase, but the losses are partly caused by the financial conditions themselves, not by unemployment alone. We model this through a single latent process $U_t$ that enters both the outcome and the macro variable additively. Formally,

\begin{assumption}[Latent confounding structure]
\label{ass:confounding}
The unobserved stationary process $\{U_t\}$ satisfies:
\begin{enumerate}[label=(\alph*), nosep]
    \item Additive confounding in outcome:
    $Y_{i,t+1} = m(I_{i,t}, A_{t+1}) + \gamma_Y U_t + \xi_{i,t+1}$,
    where $\xi_{i,t+1} \perp U_t \mid I_{i,t}, A_{t+1}$.
    \item Additive confounding in macro variable:
    $A_{t+1} = f(A_t) + \gamma_A U_t + \eta_{t+1}$,
    where $\eta_{t+1} \perp U_t$.
    \item Stationary dynamics: $\{U_t\}$ is stationary with $\E[U_t] = 0$, $\Var(U_t) = \sigma^2_U$, and autocovariance $\Cov(U_t, U_{t+k})$ exists.
    \item Independence of innovations: $\eta_{t+1} \perp \xi_{i,t+1}$.
\end{enumerate}
\end{assumption}

\begin{assumption}[Sensitivity bounds]
\label{ass:sensitivity}
$|\gamma_A| \le \bar{\gamma}_A$ and $|\gamma_Y| \le \bar{\gamma}_Y$, chosen by the analyst or calibrated from auxiliary data.
\end{assumption}

The additive confounding structure follows the sensitivity analysis literature \citep{carnegie2016sensitivity,franks2023sensitivity}. The outcome model $m$ may be arbitrarily nonlinear in observables; the additive structure restricts only how the unobserved $U_t$ enters. Exogeneity ($\gamma_A=0$ or $\gamma_Y=0$) would close the gap entirely but is implausible. Under these assumptions, we can bound the gap between $\mu_h(a)$ and $\mu^*_h(a)$:

\begin{theorem}[Partial identification]
\label{thm:partial-id}
Under Assumptions~\ref{ass:state}--\ref{ass:sensitivity}, for any macro path $a$, $\mu^*_h(a) \in [\mu_h(a) - c_h,\;\mu_h(a) + c_h]$, where
\[
c_h = \Big| \gamma_Y \sum_{j=1}^{h} \E[U_{j-1} \mid A_{1:j}=a_{1:j}] - \sum_{j=2}^{h}\big(\E[m(I^*_{i,j-1},a_j)] - \E[m(I_{i,j-1},a_j) \mid A_{1:j}=a_{1:j}]\big) \Big|,
\]
with $I^*_{i,t}$ the state under the causal distribution and $I_{i,t}$ under the predictive distribution.
\end{theorem}

The bound $c_h$ has two components. The first sum is the direct channel: at each step $j$, the confounder $U_{j-1}$ enters the outcome with coefficient $\gamma_Y$, and conditioning on $A_{1:j} = a_{1:j}$ shifts our expectation of $U_{j-1}$ away from zero. The magnitude of each shift depends on how strongly $U$ drives the macro variable ($\gamma_A$) and how extreme the stress path is. The second sum is the state propagation channel: earlier shifts of $U$ contaminate past outcomes $Y_{i,1}, \ldots, Y_{i,j-1}$, which feed into the state $I_{i,j-1}$ and propagate through $m$ into future periods. 

The state propagation channel vanishes when $m$ is affine in past outcomes; if the confounder additionally follows a Gaussian AR(1), the conditional expectations admit a closed form, as stated in the following corollary.

\begin{corollary}[Gaussian AR(1) with affine state]
\label{cor:ar1}
Add to Assumptions~\ref{ass:state}--\ref{ass:sensitivity}: (i) $U_{t+1}=\phi_U U_t + \nu_{t+1}$, $\nu_t \sim \cN(0,\sigma^2_\nu)$; (ii) $\eta_t \sim \cN(0,\sigma^2_\eta)$; (iii) $m(I_{i,t},A_{t+1}) = \beta Y_{i,t} + g(X_{i,t},A_{t+1})$ with $|\beta|<1$. Then
\[
c_h = |\gamma_Y| \cdot \Big|\textstyle\sum_{j=0}^{h-1} \beta^{h-1-j}\,\E[U_j \mid A_{1:h}=a_{1:h}]\Big|,
\]
where $\E[\mathbf{U}\mid \mathbf{r}] = \gamma_A \Sigma_{UU}(\gamma_A^2\Sigma_{UU} + \sigma_\eta^2 I)^{-1}\mathbf{r}$, with $r_j := a_j - \E[A_j \mid A_{j-1}=a_{j-1}]$ and $[\Sigma_{UU}]_{jk} = \sigma_U^2 \phi_U^{|j-k|}$.
\end{corollary}

The bound is a weighted sum of inferred confounder values. Each $\E[U_j \mid A_{1:h}=a_{1:h}]$ answers: given the macro path $a$, how much is attributable to $U$ versus idiosyncratic noise? More extreme macro residuals produce larger inferred $U_j$, so $c_h$ grows with stress severity. The weights $\beta^{h-1-j}$ reflect state propagation. When $\gamma_A=0$, the macro path carries no information about $U$ and $c_h=0$. Proof in Appendix~\ref{app:proof-ar1}.

%% ==================================================================
%% 5. ESTIMATION
%% ==================================================================

%% ==================================================================
%% 6. INFERENCE
%% ==================================================================
\section{Inference: Confidence Intervals for the Causal Estimand}
\label{sec:inf}

Our identification strategy reduces the causal estimand $\mu^*_h(a)$ to its observational counterpart $\mu_h(a)$ plus a sensitivity term $c_h$ that depends only on maintained assumptions about confounding strength. Estimating the identified set therefore reduces to estimating its center, $\mu_h(a)$. We consider two strategies. The first is recursive rollout: train a one-step predictor $\hatm$ to predict $Y_{i,t+1}$ from $(I_{i,t}, A_{t+1})$ using pre-stress data, then apply it recursively---at each step, the predicted outcome replaces the unobserved true outcome in the state, and the updated state serves as input to the next prediction (Algorithm~\ref{alg:rollout}). 

The second is direct multi-horizon estimation: for each $h$, define $m_h(\iota, a_{1:h}) := \E[Y_{i,h} \mid I_{i,0} = \iota, A_{1:h} = a_{1:h}]$ and train $\hat{m}_h$ by pooling samples across rolling origins. The two strategies trade off differently: recursive trains a single model but its error compounds with horizon; direct avoids compounding but requires a separate model per horizon. Current regulatory practice (CCAR/DFAST) is recursive. Formal error bounds for both, including a selection rule based on the amplification rate, are given in Appendix~\ref{app:proofs-est}.

Either strategy yields a point estimate $\hatmu_h(a)$ of the predictive mean $\mu_h(a)$. The error bounds in Appendix~\ref{app:proofs-est} guarantee that estimation error is finite and independent of the confounding parameters, but depend on worst-case quantities (e.g., uniform one-step error $\epsilon_n$) that are too conservative for practical use. We instead quantify estimation uncertainty from data using conformal prediction \citep{vovk2005algorithmic}, which is distribution-free, compatible with any base learner, and yields substantially tighter intervals.

\subsection{Weighted Conformal Calibration}
\label{sec:inf-est}

The idea of conformal prediction is simple: if we know how large the estimator's prediction errors were on historical data, we can use that to bound how large the error will be on the stress scenario. The complication is that historical macro paths differ from the stress path, so we importance-weight the historical errors to reflect what they would look like under stress conditions.

Fix a target horizon $h$.  Split the pre-stress data into a training segment (used to fit $\hat{m}$) and a held-out calibration segment.  Within the calibration segment, select $B$ rolling origins spaced $g$ periods apart;\footnote{The gap $g$ balances two considerations: larger $g$ reduces temporal dependence between consecutive scores, but yields fewer calibration points $B \approx \lfloor(T_{\mathrm{cal}} - h)/g\rfloor$.} each origin $b$ must have at least $h$ subsequent periods of observed data.  At each origin $b$, run the rollout $h$ steps forward along the realized macro path and record the prediction error
\[
S_b := \big|N^{-1}\textstyle\sum_i \big(\hat{Y}_{i,h}(b) - Y_{i,h}(b)\big)\big|,
\]
where $\hat{Y}_{i,h}(b)$ is the $h$-step rollout prediction from origin $b$ and $Y_{i,h}(b)$ is the corresponding realized outcome. Each calibration score $S_b$ is computed under the historical macro path from origin $b$, but the quantity of interest is the prediction error under the stress path $a^S = (a^S_1, \ldots, a^S_H)$.  To bridge this distribution shift, assign each score an importance weight
\[
w_b(a^S) := \prod_{j=1}^{h} \frac{p(a^S_j \mid J_{b,j-1})}{p(a_{b,j} \mid J_{b,j-1})},
\]
where $a_{b,j}$ is the realized macro value at the $j$-th step from origin $b$ and $J_{b,j}$ is the aggregate macro state at that point.  The conditional density $p(\cdot \mid J_{b,j-1})$ is estimated from pre-stress data using any consistent density model for $A_{t+1}$ given the macro state.

The weighted scores $\{(S_b, w_b)\}_{b=1}^B$ define an empirical distribution of prediction errors reweighted to approximate stress conditions.  The estimation uncertainty band is the weighted conformal quantile
\[
q^w_{1-\alpha} := \inf\Big\{q : \frac{\sum_b w_b \mathbf{1}\{S_b \le q\} + W_{\max}}{\sum_b w_b + W_{\max}} \ge 1 - \alpha\Big\},
\]
where $W_{\max} := \max_b w_b$ which acts as a conservative correction: it accounts for the possibility that the stress period's true prediction error exceeds all observed historical errors, ensuring the coverage guarantee holds even in this worst case.

%% --- Part 3: Guarantee ---
Standard conformal prediction guarantees exact $1-\alpha$ coverage under exchangeability, which our setting violates in two ways: calibration scores are temporally dependent, and the stress macro path differs from the historical paths used for calibration. Importance weighting \citep{tibshirani2019conformal} and spacing out rolling origins \citep{barber2023conformal} address these violations, but neither correction is exact, so the coverage guarantee degrades by two quantifiable terms: $\Rmix$ for residual temporal dependence and $\Rweight$ for residual extrapolation cost. The formal guarantee requires three additional assumptions:

\begin{assumption}[Decaying temporal dependence]
\label{ass:mixing}
The joint process $\{(Y_{i,t}, X_{i,t}, A_t)\}_t$ is $\beta$-mixing with mixing coefficients $\beta(k) \le C_\beta e^{-\lambda k}$.
\end{assumption}

\begin{assumption}[Bounded likelihood ratio]
\label{ass:bounded-lr}
$w_b(a^S) \le W_{\max} < \infty$ for all $b$.
\end{assumption}

\begin{assumption}[Score path-homogeneity]
\label{ass:homogeneity}
The distribution of the rollout error score, conditional on $J_{b,j}$, depends on the macro path only through the likelihood ratio.
\end{assumption}

Assumption~\ref{ass:mixing} requires that temporal dependence decays exponentially, which is standard for macroeconomic data.  Assumption~\ref{ass:bounded-lr} ensures the stress path is not so extreme that no historical period receives meaningful weight---the formal version of the regulatory requirement that stress scenarios be ``severe but plausible.''  Assumption~\ref{ass:homogeneity} requires that the model's prediction error behavior is stable across macro environments after conditioning on the aggregate state, so that importance weighting correctly transfers calibration information from historical to stress periods.

\begin{proposition}[Weighted calibration coverage]
\label{prop:weighted-cal}
Under Assumptions~\ref{ass:mixing}--\ref{ass:homogeneity}:
\begin{equation}
\label{eq:cal-coverage}
P\!\big(\mu_h(a) \in [\hat{\mu}_h(a) \pm q^w_{1-\alpha}]\big) \ge 1 - \alpha - \Rmix - \Rweight, \nonumber
\end{equation}
where $\Rmix = 2(B+1)\beta(g)$ captures temporal dependence ($g$ is the gap between consecutive rolling origins) and $\Rweight = W_{\max}/(\sum_b w_b + W_{\max})$ captures extrapolation cost.
\end{proposition}

Combining Proposition~\ref{prop:weighted-cal} (bounding $|\hatmu_h(a)-\mu_h(a)|$) with Theorem~\ref{thm:partial-id} (bounding $|\mu_h(a)-\mu^*_h(a)|$) by triangle inequality allows us to compute the bound empirically:

\begin{theorem}[CI for $\mu^*_h(a)$]
\label{thm:ci}
Under Assumptions~\ref{ass:state}--\ref{ass:homogeneity} and $|\mu^*_h(a)-\mu_h(a)|\le c_h$,
\[
\mathrm{CI}_h(a) := \big[\hatmu_h(a) \pm (q^w_{1-\alpha} + c_h)\big], \qquad \Pbb\big(\mu^*_h(a)\in\mathrm{CI}_h(a)\big)\ge 1-\alpha-\Rmix-\Rweight.
\]
\end{theorem}

The CI decomposes into two layers: estimation ($q^w_{1-\alpha}$, data-driven, adapts to stress severity) and identification ($c_h$, analyst-specified via $(\bar{\gamma}_A,\bar{\gamma}_Y)$). Coverage is conditional on correct specification of $(\bar\gamma_A,\bar\gamma_Y)$---the standard sensitivity-analysis condition \citep{rosenbaum2002observational,cinelli2020sensemakr}. Proofs are in Appendices~\ref{app:proof-weighted-cal} and~\ref{app:proof-ci}.

%% ==================================================================
%% 7. EXPERIMENTS
%% ==================================================================
\section{Semi-Synthetic Experiments}
\label{sec:experiments}

The theory establishes an irreducible gap $c_h$ between the predictive mean $\mu_h(a)$ and the causal mean $\mu^*_h(a)$. The experiments ask three questions. First, does this gap matter in practice -- can a model that passes standard backtesting validation still be systematically wrong about the causal quantity? Second, does the framework close it -- does adding the confounding envelope $c_h$ to the weighted conformal band restore nominal coverage of $\mu^*_h$? Third, can practitioners anchor the sensitivity parameters $(\bar\gamma_A, \bar\gamma_Y)$ in observable data, or are they purely hypothetical?

We adopt a semi-synthetic design in order to evaluate results against known counterfactual ground truth, which is not available in real data. We summarize the key settings below; more details of experiments are in Appendix~\ref{app:experiment_details}. We build the macro variable $A_t$ using historical U.S.\ unemployment rates\footnote{Monthly civilian unemployment rate from FRED (series UNRATE), \url{https://fred.stlouisfed.org/series/UNRATE}, 1990--2024, $T=420$ months. We acknowledge that the observed unemployment rate is itself influenced by latent factors, but it provides a realistic approximation of macro dynamics for our simulation.}, while the confounder $U_t$ and outcomes $Y_{i,t}$ are synthetic with known parameters. This inherits real macro dynamics, including the 2001, 2008, and 2020 recessions. We assume the confounder $U_t$ follows an AR(1) with Gaussian innovations. The learner observes the confounded macro variable and outcomes but not the confounder itself. We consider two specifications for $m$ in DGPs:
\begin{align*}
    \text{Linear:} \quad
    m(Y_{i,t}, X_i, A_{t+1}) &= \alpha + \beta_1 Y_{i,t} + \beta_2 A_{t+1} + \beta_3 X_i, \\
    \text{Nonlinear:} \quad
    m(Y_{i,t}, X_i, A_{t+1}) &= \alpha + \beta_1 Y_{i,t} + \beta_2 A_{t+1}
        + \beta_4 A_{t+1}^2
        + \beta_5 \max(Y_{i,t} - \bar{Y},\, 0) + \beta_3 X_i.
\end{align*}
The linear model is the standard specification in regulatory stress testing. The nonlinear model adds a quadratic macro response and a threshold amplifier. The stress scenario is a flat unemployment path at approximately 9\%, comparable in severity to the 2008 financial crisis. We use OLS for the linear DGP and a two-layer MLP for the nonlinear DGP. The sample is divided into three non-overlapping windows. The model is trained on 1990--2014, conformal calibration scores are computed on 2015--2019 without updating model parameters, and the stress rollout is evaluated over a 12-month horizon from 2020.

\begin{table}[t]
\centering
\footnotesize
\setlength{\tabcolsep}{3pt}
\renewcommand{\arraystretch}{0.85}
\caption{DGP parameters. Left: outcome model. Right-top: confounder (synthetic). Right-bottom: macro (fitted from unemployment rate).}
\label{tab:dgp-params-app}
\begin{tabular}{@{}lll|lll@{}}
\toprule
Symbol & Value & Description & Symbol & Value & Description \\
\midrule
$\alpha$ & 1.0 & Intercept & $\phi_U$ & 0.85 & Confounder persistence \\
$\beta_1$ & 0.85 & Outcome persistence & $\sigma_\nu$ & 0.5 & Confounder innov.\ std \\
$\beta_2$ & 0.15 & Macro sensitivity & $(\gamma_A,\gamma_Y)$ & varies & Confounding strength \\
$\beta_3$ & 0.3 & Cross-sect.\ covariate & \multicolumn{3}{c}{} \\
$\beta_4$ & 0.02 & Quadratic (nonlin.) & $\hat\phi_A$ & 0.995 & AR(1) coeff.\ (fitted) \\
$\beta_5$ & 0.12 & Threshold (nonlin.) & $\hat\mu_A$ & 5.84 & Mean unemp.\ rate (\%) \\
$\bar{Y}$ & 18.0 & Distress thresh.\ (nonlin.) & $\hat\sigma_\eta$ & 0.154 & AR(1) innov.\ std \\
$\sigma_\xi$ & 0.15 & Idiosyncratic noise & $\hat\sigma_A$ & 1.59 & Marginal std (\%) \\
$N$ & 3{,}000 & Cross-sect.\ units & $a^S$ & 9.03 & Stress: $\hat\mu_A+2\hat\sigma_A$ (\%) \\
$H$ & 12 & Horizon (months) & & & \\
\bottomrule
\end{tabular}
\end{table}

\subsection{Does Confounding Cause Silent Failure?}
\label{sec:exp-q1}

We run three experiments to show that a model passing standard validation can be arbitrarily wrong about $\mu^*_h(a^S)$. The first computes $c_h$ analytically with fixed $\gamma_A$ and varying $\gamma_Y$, establishing that the gap is structural. The second computes the stress test outcome with and without the latent factor's expected contribution, showing the gap dwarfs visible estimation uncertainty. The third compares out-of-sample backtesting $R^2$ against $c_h$ across confounding strengths, demonstrating that standard validation metrics cannot detect the bias.

Figure~\ref{fig:q1} summarizes all three results. The top panel shows the confounding gap $c_h$ computed in closed form, with no model training: the gap grows linearly with confounding strength, reaching 1.23 loss units at $\gamma_Y = 1.0$. The middle panel shows that the two answers diverge immediately, yet the $\pm 2\,$SE band around the predictive mean is too narrow to see at this scale---the hidden bias exceeds visible uncertainty by 58--212$\times$ across horizons. The bottom panel shows that one-step backtesting $R^2$ remains flat at $\sim$0.957 regardless of confounding strength, while $c_{12}$ jumps from zero to 0.6---an analyst inspecting standard validation metrics cannot distinguish a confounded world from an unconfounded one. Together, the three experiments show that the confounding gap is structural, large relative to estimation uncertainty, and invisible to standard model validation---precisely the conditions under which silent failure occurs in practice.

\begin{figure}[H]
\centering
\includegraphics[width=\textwidth, height=2.8cm, keepaspectratio=false]{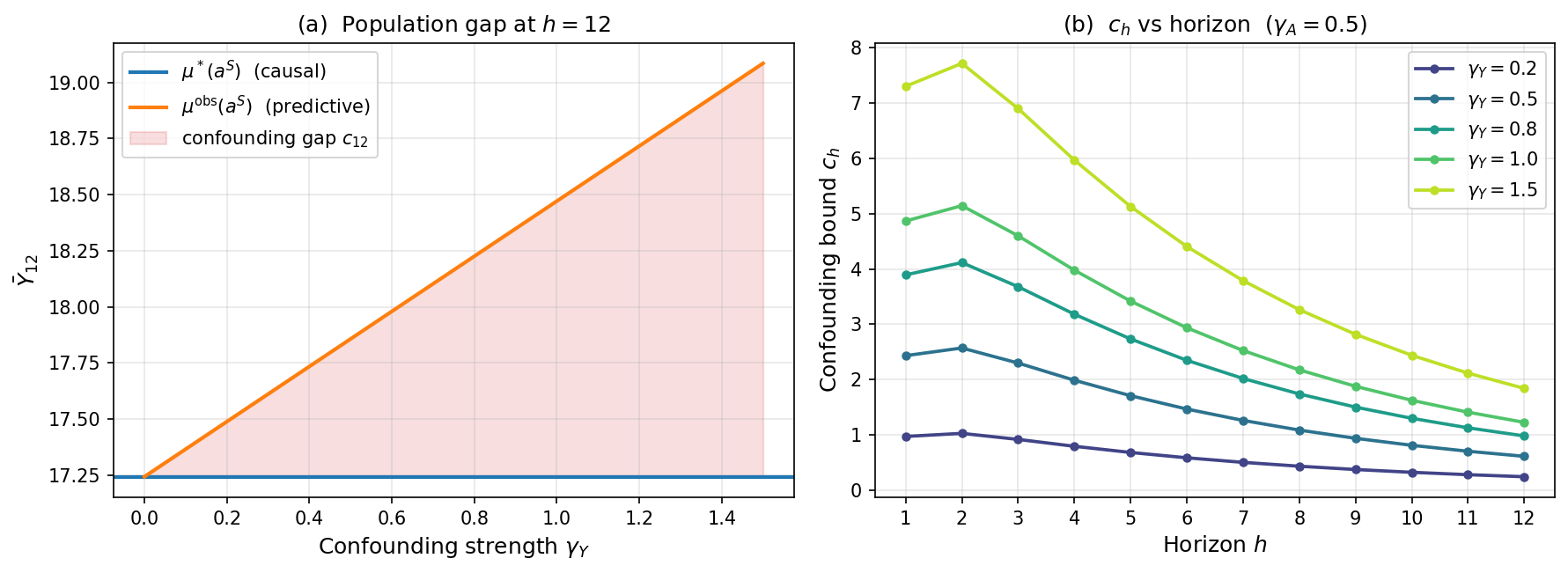}\\[1pt]
\includegraphics[width=\textwidth, height=2.8cm, keepaspectratio=false]{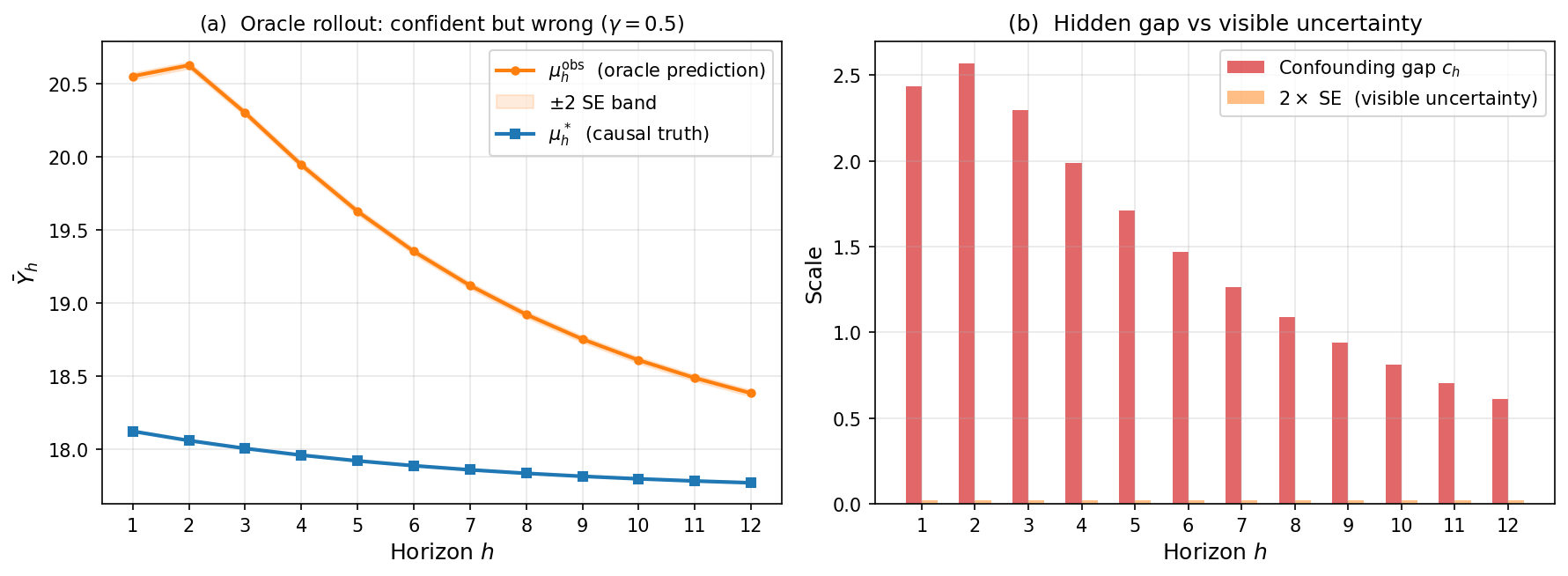}\\[1pt]
\includegraphics[width=\textwidth, height=2.8cm, keepaspectratio=false]{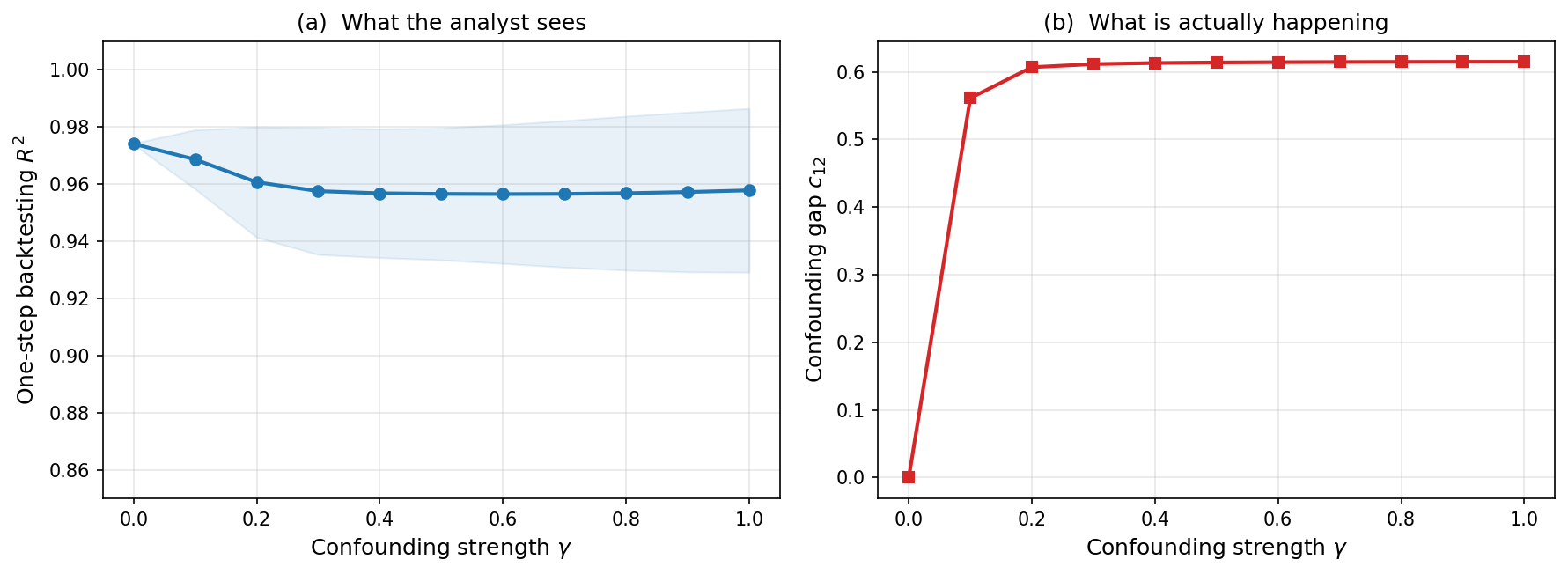}
\caption{Silent failure of predictive stress testing. Top: Population-level confounding bias ($\gamma_A=0.5$); $\mu^*_{12}$ is invariant to $\gamma_Y$ while $\mu_{12}$ rises linearly. Middle: Oracle rollout at $\gamma=0.5$; causal truth $\mu^*_h$ lies 58--212 SE outside the predictive band. Bottom: Backtesting $R^2$ is flat (0.974$\to$0.957) while $c_{12}$ jumps from 0 to 0.61.}
\label{fig:q1}
\end{figure}

\subsection{Does the Confounding Envelope Restore Coverage?}
\label{sec:exp-q2a}

We test whether our estimator can cover $\mu^*_h$ despite not modeling confounding. After training, we pick a starting point in 2015--2019 and let the model predict forward step by step using actual unemployment. At each step, we record the gap between the model's predicted portfolio loss and the actual portfolio loss. Twelve steps yield twelve errors, filed into separate bins for $h=1$ through $h=12$. We repeat from different starting points to accumulate a collection of historical errors for each horizon, then use each horizon's errors to calibrate the corresponding interval width.

We compare our methods with three benchmarks commonly used in the field. First, ordinary least squares, which constructs intervals from compounded regression standard errors across the rollout. Second, bootstrap, which resamples borrowers, retrains, and takes percentile intervals. Third, EnbPI \citep{XuXie2021}, which uses the per-horizon historical errors directly as interval widths. We also compare two versions of our method: one reweights those errors by importance weights to account for the distributional shift to the stress scenario, and the other adds the confounding envelope $c_h$ on top of the reweighted interval.

Figure~\ref{fig:q2a-heatmap} shows coverage at $\gamma=0.4$, nominal 90\%, 50 replications. On the linear DGP (left), three tiers emerge: OLS and bootstrap achieve 0--8\%, EnbPI and ours~(cal) reach 82--98\%, and only ours~(full) achieves 100\% across all 12 horizons. The decisive comparison is the bottom two rows: same conformal machinery, same weights---the jump from 82--98\% to 100\% is entirely from $c_h$. On the nonlinear DGP (right), baseline coverage at $\gamma=0$ is already 54--78\%, reflecting MLP extrapolation failure under stress---an estimation issue orthogonal to identification. Even so, ours~(full) adds 44 percentage points at $h=1$ (52\%$\to$96\%) and remains the only method with coverage $\ge 90\%$ at any horizon. The results confirm that the confounding envelope is necessary and sufficient for causal coverage when the base predictor is well-specified, and provides substantial gains even when it is not. Estimation failure and identification failure are independent---each requires its own remedy.

\begin{figure}[h]
\centering
\begin{minipage}{0.49\textwidth}
    \centering
    \includegraphics[width=\textwidth]{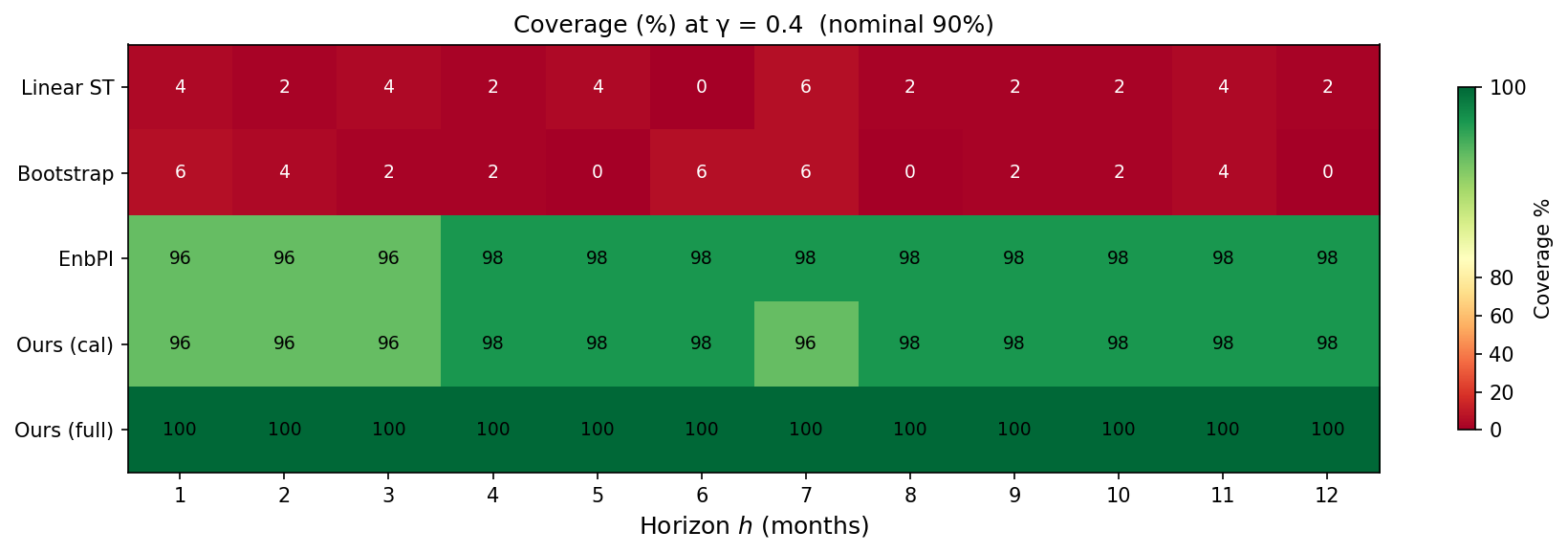}
\end{minipage}
\hfill
\begin{minipage}{0.49\textwidth}
    \centering
    \includegraphics[width=\textwidth]{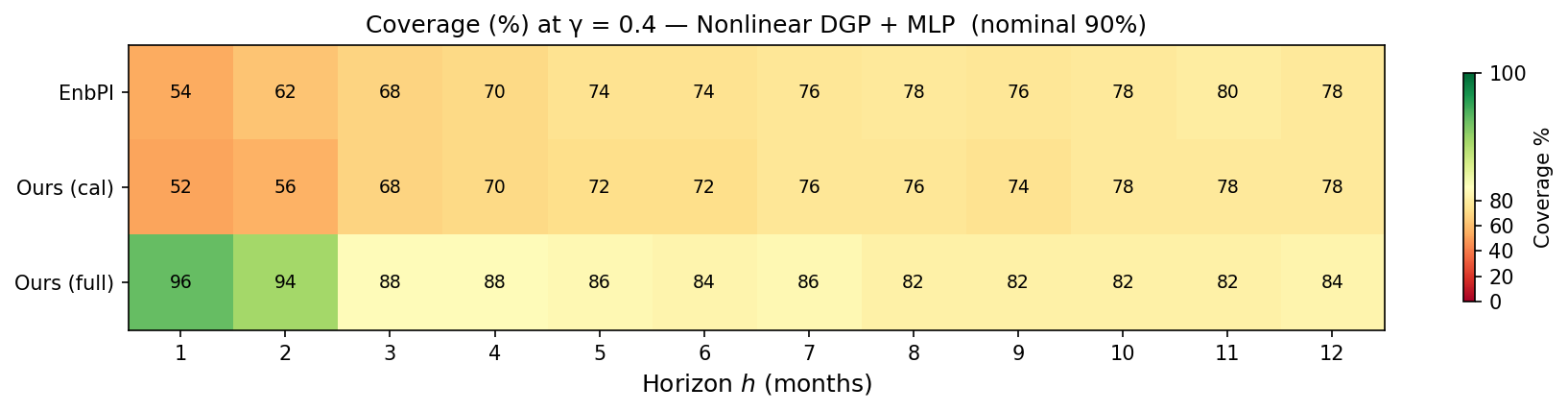}
\end{minipage}
\caption{Coverage heatmaps at $\gamma=0.4$, nominal 90\%.
Left: linear DGP + OLS (5 methods).
Right: nonlinear DGP + MLP (3 methods).
Ours~(full) is the only method achieving 100\% across
all horizons (linear) or $\geq 90\%$ at short horizons (nonlinear).}
\label{fig:q2a-heatmap}
\end{figure}

\subsection{Two-Layer Decomposition}
\label{sec:exp-decomposition}

The heatmaps show that the confounding envelope matters; Figure~\ref{fig:decomposition} shows how much. At each horizon we decompose the interval half-width into two components: the historical prediction error (from the calibration step) and the confounding correction $c_h$. In the linear DGP, $c_h$ accounts for 78\% of the total at $h=1$ and 15\% at $h=12$; in the nonlinear DGP, 67\% and 7.5\%. The $c_h$ bars are identical across panels ---the identification correction depends only on structural parameters, not the learner. At short horizons, correcting for confounding is the dominant concern; at long horizons, improving the base predictor matters more. This directly mirrors the three-layer decomposition in the theory: estimation and identification contribute independently, and the practitioner can diagnose which layer is binding at any given horizon.

\begin{figure}[H]
\centering
\includegraphics[width=0.75\textwidth]{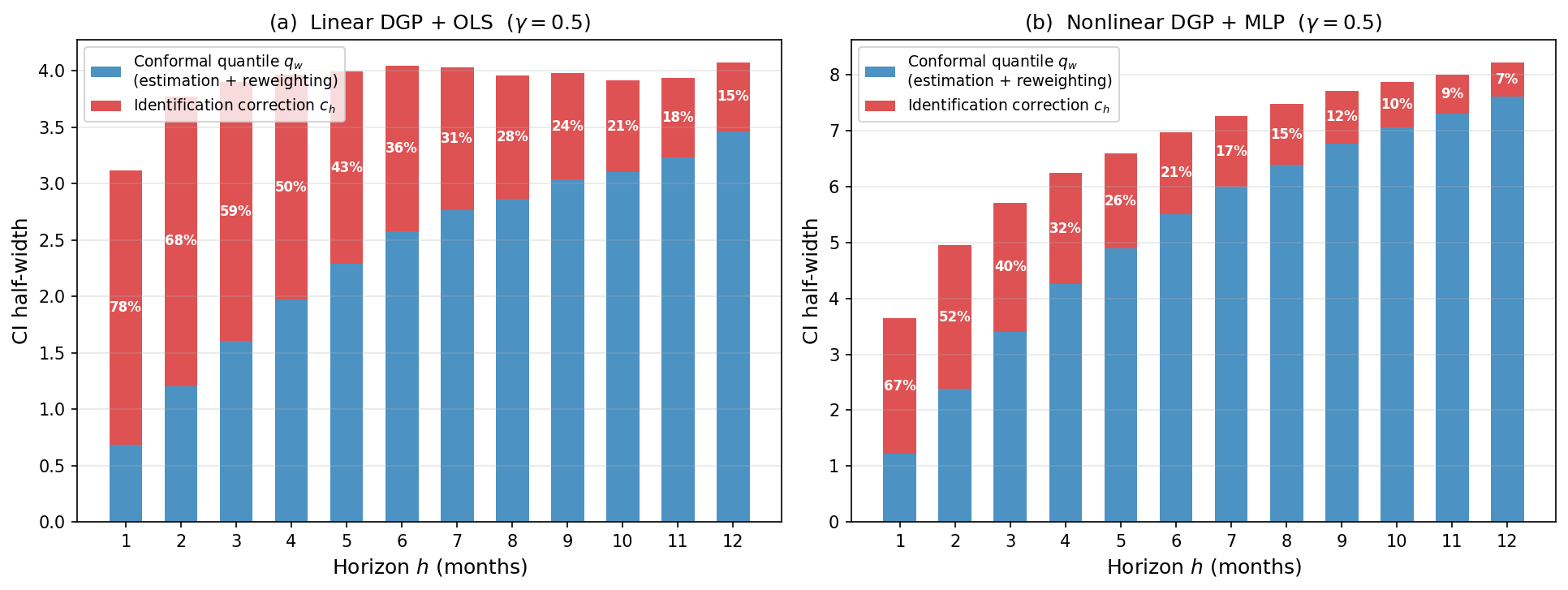}
\caption{CI half-width decomposition at $\gamma=0.5$. Blue: $q^w_h$ (estimation + reweighting). Red: $c_h$ (identification). (a) Linear DGP: $c_h$ 78\% at $h=1$. (b) Nonlinear DGP: same $c_h$, larger $q^w_h$.}
\label{fig:decomposition}
\end{figure}

\subsection{Empirical Calibration via NFCI}
\label{sec:exp-q3a}

In this section, we present when working on regulatory stress testing, how a practitioner calibrate confounder parameters in observable data. We use the Chicago Fed National Financial Conditions Index (NFCI) as a proxy for the latent confounder: fitting an AR(1) to monthly NFCI (1990--2019) gives $\hat\phi_U=0.973$ and $\hat\sigma_\nu=0.124$, and regressing UNRATE residuals on lagged NFCI yields $\hat\gamma_A=0.091$. Setting $\gamma_Y=\hat\gamma_A$ produces a fully calibrated parameter set derived entirely from observable data.

We re-examine the coverage with calibrated confounder strength. The calibrated confounding is modest ($\gamma=0.091$), but the results are decisive (Figure~\ref{fig:q3a}). On the linear DGP, OLS intervals achieve 2--8\% coverage---even mild confounding causes complete failure. Ours~(full) achieves 100\% across all horizons. On the nonlinear DGP, EnbPI and ours~(cal) reach only 22--70\%, while ours~(full) ranges from 74\% at $h=1$ to 100\% at $h=9$. The gap between ours~(cal) and ours~(full) is 52 percentage points at $h=1$: even a small, empirically calibrated confounding correction has a large coverage impact.

\begin{figure}[H]
\centering
\begin{minipage}{0.49\textwidth}
\centering
\includegraphics[width=\textwidth]{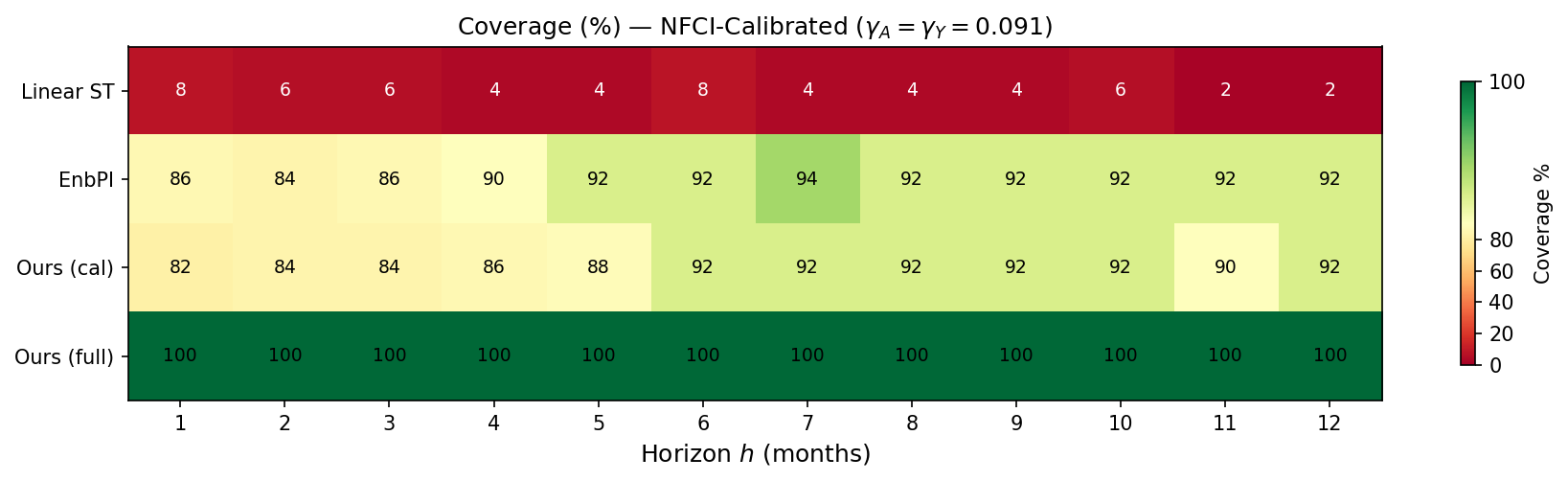}
\end{minipage}\hfill
\begin{minipage}{0.49\textwidth}
\centering
\includegraphics[width=\textwidth]{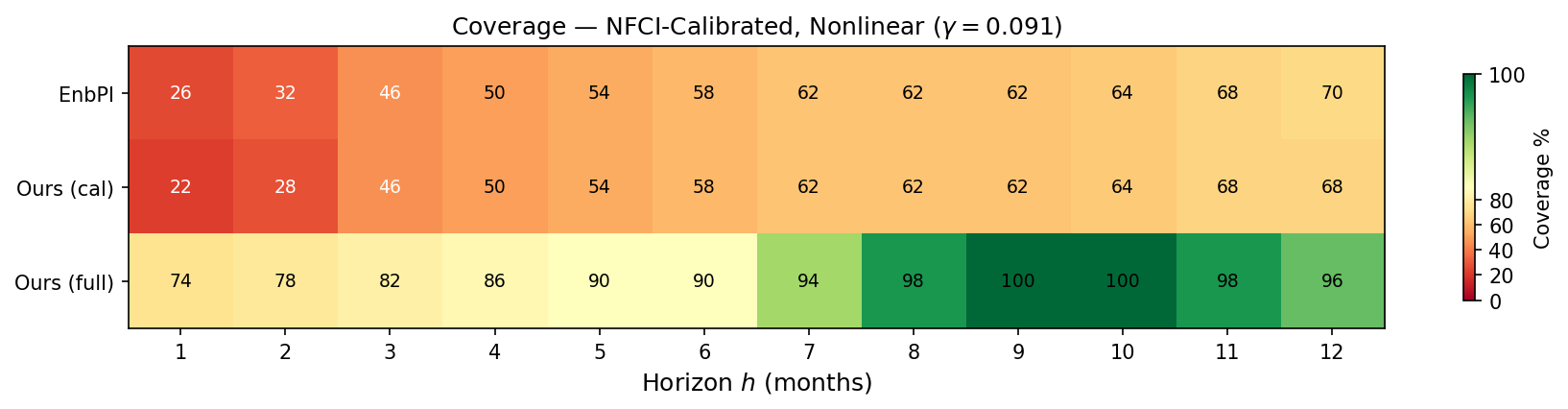}
\end{minipage}
\caption{NFCI-calibrated coverage ($\gamma_A=\gamma_Y=0.091$, nominal 90\%). \textbf{Left}: linear DGP + OLS. \textbf{Right}: nonlinear DGP + MLP. All confounder parameters estimated from NFCI. Ours~(full) is the only method with coverage $\ge 90\%$ under NFCI-calibrated confounding.}
\label{fig:q3a}
\end{figure}

%% ==================================================================
%% 8. DISCUSSION
%% ==================================================================

\section{Conclusion}
\label{sec:conclusion}

Regulatory stress testing asks a causal question but is answered by predictive methods, and the gap is quantitatively large: under NFCI-calibrated confounding, standard uncertainty methods cover the causal stress-period mean 0--8\% of the time versus a nominal 90\%. We have presented a framework that closes this gap through three modular pieces: a closed-form partial identification bound $c_h$, non-asymptotic error control for the industry-standard recursive rollout, and importance-weighted conformal calibration with abstention diagnostics. The combined interval $\hatmu_h(a)\pm(q^w_{1-\alpha}+c_h)$ separates estimation from identification uncertainty, each with its own guarantee, giving practitioners and regulators a transparent basis for decision-making. The framework does not eliminate uncertainty---it makes uncertainty visible and interpretable.

%% ==================================================================
%% REFERENCES
%% ==================================================================
\bibliographystyle{plainnat}
\bibliography{references}

%% ==================================================================
%% APPENDIX
%% ==================================================================
\newpage
\appendix

%% ------------------------------------------------------------------
\section{Proofs for Identification}
\label{app:proofs-id}

\subsection{Estimation of the Predictive Mean}
\label{app:proof-obs-id}

We formalize the claim in Section~\ref{sec:setup} that $\mu_h(a)$ is estimable from pre-stress data.

\begin{proposition}[Estimation of predictive mean]
\label{prop:obs-id}
Under Assumptions~\ref{ass:state}--\ref{ass:stationary}, the predictive mean $\mu_h(a)$ can be estimated from pre-stress data.
\end{proposition}

\begin{proof}
The initial state $I_{i,0}$ is observed. At $h=1$: by Assumption~\ref{ass:stationary}, the distribution of $Y_{i,1}$ given $I_{i,0}$ and $A_1=a_1$ is the same as in pre-stress data, so $\mu_1(a)$ is estimable. This also determines the distribution of $Y_{i,1}$ under $a_1$, which, together with $I_{i,0}$ and $a_1$, determines the distribution of $I_{i,1}=g(\cH_{i,1})$. At $h=2$: by Assumption~\ref{ass:state}, $Y_{i,2}$ depends on history only through $I_{i,1}$ and $A_2=a_2$. The distribution of $I_{i,1}$ is known from the previous step; stationarity applies again, so $\mu_2(a)$ is estimable. Induction to horizon $H$ completes the proof.
\end{proof}

\subsection{Proof of Theorem~\ref{thm:partial-id}}
\label{app:proof-partial-id}

\begin{proof}
Define $b_j := \E[U_{j-1}\mid A_{1:j}=a_{1:j}]$, and let $I^*_{i,t},I_{i,t}$ denote the state under the causal and predictive distributions.

\textbf{$h=1$.} By Assumption~\ref{ass:confounding}(a), $Y_{i,1}=m(I_{i,0},a_1)+\gamma_Y U_0 + \xi_{i,1}$. Since $I_{i,0}$ is observed (same under both distributions):
\[
\mu^*_1(a) = \E[m(I_{i,0},a_1)] + \gamma_Y\cdot\E[U_0] = \E[m(I_{i,0},a_1)], \quad
\mu_1(a) = \E[m(I_{i,0},a_1)] + \gamma_Y\cdot b_1.
\]
The causal mean uses $\E[U_0]=0$ (setting $A_1$ externally severs the link $U_0\to A_1$); the predictive mean uses $\E[U_0\mid A_1=a_1]=b_1$ (conditioning on $A_1=f(A_0)+\gamma_A U_0+\eta_1$ shifts $U_0$, by Assumption~\ref{ass:confounding}(b)). Gap: $\Delta_1(a)=-\gamma_Y b_1$. $b_1$ is finite by Assumption~\ref{ass:confounding}(c); $|\gamma_Y|\le\bar\gamma_Y$ by Assumption~\ref{ass:sensitivity}; so $|\Delta_1(a)|\le\bar\gamma_Y\sigma_U<\infty$.

The state then updates: $I_{i,1}$ includes $Y_{i,1}$, so $I^*_{i,1}$ and $I_{i,1}$ differ through $\gamma_Y U_0$ inside $Y_{i,1}$.

\textbf{$h=2$.} $Y_{i,2}=m(I_{i,1},a_2)+\gamma_Y U_1+\xi_{i,2}$, so
\[
\mu^*_2(a) = \E[m(I^*_{i,1},a_2)], \quad \mu_2(a) = \E[m(I_{i,1},a_2)\mid A_{1:2}=a_{1:2}] + \gamma_Y b_2.
\]
The gap has two terms: $\Delta_2(a) = D_2 - \gamma_Y b_2$ with $D_2 := \E[m(I^*_{i,1},a_2)] - \E[m(I_{i,1},a_2)\mid A_{1:2}=a_{1:2}]$. $D_2$ is state propagation. All terms are finite.

\textbf{General $h$.} Define $D_j := \E[m(I^*_{i,j-1},a_j)] - \E[m(I_{i,j-1},a_j)\mid A_{1:j}=a_{1:j}]$ with $D_1:=0$. Then $\Delta_h(a) = \sum_{j=1}^h D_j - \gamma_Y\sum_{j=1}^h b_j$. Each $b_j$ is finite (Assumption~\ref{ass:confounding}(c)); each $D_j$ is a difference of finite expectations; the sum has $h<\infty$ terms. Hence $|\Delta_h(a)| =: c_h < \infty$ and $\mu^*_h(a)\in[\mu_h(a)-c_h,\mu_h(a)+c_h]$.
\end{proof}

\subsection{Proof of Corollary~\ref{cor:ar1}}
\label{app:proof-ar1}

\begin{proof}
By condition (iii), $Y_{i,t+1} = \beta Y_{i,t} + g(a_{t+1}) + \gamma_Y U_t + \xi_{i,t+1}$. Iterating from $Y_{i,0}$:
\[
Y_{i,h} = \beta^h Y_{i,0} + \sum_{j=0}^{h-1}\beta^{h-1-j}[g(a_{j+1}) + \gamma_Y U_j + \xi_{i,j+1}].
\]
\emph{Causal mean.} Setting $A_{1:h}=a_{1:h}$ externally severs $U\to A$; $\E[U_j]=0$, $\E[\xi_{i,j+1}]=0$:
$\mu^*_h(a) = \beta^h\E[Y_{i,0}] + \sum_{j=0}^{h-1}\beta^{h-1-j}g(a_{j+1})$.

\emph{Predictive mean.} Conditioning on $A_{1:h}=a_{1:h}$ shifts each $U_j$; since $Y_{i,h}$ is linear in $(U_0,\ldots,U_{h-1})$:
$\mu_h(a) = \beta^h\E[Y_{i,0}] + \sum_{j=0}^{h-1}\beta^{h-1-j}[g(a_{j+1}) + \gamma_Y\E[U_j\mid A_{1:h}=a_{1:h}]]$.

\emph{Gap.} The $g$ and $Y_{i,0}$ terms cancel:
$\Delta_h(a) = -\gamma_Y\sum_{j=0}^{h-1}\beta^{h-1-j}\E[U_j\mid A_{1:h}=a_{1:h}]$.

\emph{Closed form.} By conditions (i) and (ii), $(\mathbf{U},\mathbf{r})$ is jointly Gaussian. Since $r_j=\gamma_A U_{j-1}+\eta_j$:
$\Cov(\mathbf{U},\mathbf{r})=\gamma_A\Sigma_{UU}$, $\Var(\mathbf{r})=\gamma_A^2\Sigma_{UU}+\sigma_\eta^2 I$.
Gaussian conditioning gives $\E[\mathbf{U}\mid\mathbf{r}]=\gamma_A\Sigma_{UU}(\gamma_A^2\Sigma_{UU}+\sigma_\eta^2 I)^{-1}\mathbf{r}$. Substituting and taking absolute values yields $c_h$.
\end{proof}

%% ------------------------------------------------------------------
\section{Proofs for Estimation}
\label{app:proofs-est}

To formalize the tradeoff between two estimators, we need two regularity conditions: the outcome model and state transition are smooth, and the learner achieves uniform one-step error.

\begin{assumption}[Smooth state dynamics]
\label{ass:lipschitz}
Let $\iota, \iota'$ denote two states, $a$ the macro value, $y, y'$ realized outcomes, and $\iota^+, \iota'^+$ the corresponding next-period states.  There exist constants $L_m, L_u \ge 0$ such that:
\begin{enumerate}[label=(\roman*), nosep]
    \item Outcome model: $|m(\iota, a) - m(\iota', a)| \le L_m \|\iota - \iota'\|$.
    \item State transition: Let $\iota^+$ and $\iota'^+$ be the next-period states obtained by updating $\iota$ with outcome $y$ and $\iota'$ with outcome $y'$, respectively, under the same macro value $a$.  Then $\|\iota^+ - \iota'^+\| \le L_u(\|\iota - \iota'\| + |y - y'|)$.
\end{enumerate}
\end{assumption}

\begin{assumption}[One-step excess risk]
\label{ass:onestep}
With probability $\ge 1 - \delta_n$: $\sup_{\iota, a} |\hat{m}(\iota, a) - m(\iota, a)| \le \epsilon_n$.
\end{assumption}

When $I_{i,t}=(Y_{i,t},Y_{i,t-1},\ldots)$ is a finite lag vector, the update shifts entries by one, so Assumption~\ref{ass:lipschitz}(ii) holds with $L_u=1$. The \emph{amplification rate} $\rho := L_u(1+L_m)$ governs how one-step errors compound: if $\rho < 1$, errors shrink with each step; if $\rho > 1$, they grow exponentially. The following theorem makes this precise.

\begin{theorem}[Recursive estimation error]
\label{thm:oracle}
Under Assumptions~\ref{ass:lipschitz}--\ref{ass:onestep}, with $\rho := L_u(1+L_m)$ and $\Gamma_h := \sum_{j=0}^{h-1}\rho^j$,
\[
|\hatmu_h(a)-\mu_h(a)| \le \underbrace{\epsilon_n \Gamma_h}_{\text{learner}} + \underbrace{|b_h(a)|}_{\text{mean-state bias}} + \underbrace{O_p(N^{-1/2})}_{\text{sampling}},
\]
where $b_h(a)$ is zero for affine $m$.
\end{theorem}

Proof in Appendix~\ref{app:proof-oracle}. The first term compounds one-step errors through the rollout; $\Gamma_h \to 1/(1-\rho)$ when $\rho<1$ and grows exponentially when $\rho>1$. The second arises because the rollout feeds means back into $m$ rather than averaging over the stochastic outcome distribution. The third is cross-sectional noise. On the other side, the direct estimator avoids this compounding entirely.

\begin{corollary}[Direct estimator error]
\label{cor:direct}
If $\sup_{\iota,a_{1:h}}|\hatm_h(\iota,a_{1:h})-m_h(\iota,a_{1:h})|\le \epsilon^{\mathrm{dir}}_{n,h}$ with prob.\ $\ge 1-\delta_n$, then $|\hatmu_h(a)-\mu_h(a)| \le \epsilon^{\mathrm{dir}}_{n,h} + O_p(N^{-1/2})$.
\end{corollary}

The amplification rate $\rho$ provides a practical rule. When $\rho<1$, $\Gamma_h \to 1/(1-\rho)$ stabilizes and recursive is reliable. When $\rho>1$, switch to direct beyond $h^* := \min\{h:\epsilon_n\Gamma_h > \epsilon^{\mathrm{dir}}_{n,h}\}$. When $\rho\approx 1$, compare empirically. In practice $\rho$ need not be computed analytically---it suffices to observe whether recursive rollout errors grow or stabilize with horizon.

\subsection{Proof of Theorem~\ref{thm:oracle}}
\label{app:proof-oracle}

\begin{proof}
Define the \emph{true-model rollout} $\barY_{i,j}$: starting from $I_{i,0}$, iterate $\barY_{i,j}=m(\barI_{i,j-1},a_j)$ and $\barI_{i,j}$ via the same update but with $\barY_{i,j}$ in place of the true outcome. By definition $\hatmu_h(a)=N^{-1}\sum_i\hatY_{i,h}$; adding/subtracting $\barY_{i,h}$ and $\E[\barY_{i,h}]$:
\[
\hatmu_h(a)-\mu_h(a) = \underbrace{N^{-1}\textstyle\sum_i(\hatY_{i,h}-\barY_{i,h})}_{\text{(A)}} + \underbrace{N^{-1}\textstyle\sum_i \barY_{i,h}-\E[\barY_{i,h}]}_{\text{(B)}} + \underbrace{\E[\barY_{i,h}]-\mu_h(a)}_{\text{(C)}}.
\]

\textbf{(A) Learner error.} At initialization, $\hatI_{i,0}=\barI_{i,0}=I_{i,0}$. At step $j$:
\begin{align*}
|\hatY_{i,j}-\barY_{i,j}| &= |\hatm(\hatI_{i,j-1},a_j) - m(\barI_{i,j-1},a_j)| \\
&\le \underbrace{|\hatm(\hatI_{i,j-1},a_j)-m(\hatI_{i,j-1},a_j)|}_{\le \epsilon_n} + \underbrace{|m(\hatI_{i,j-1},a_j)-m(\barI_{i,j-1},a_j)|}_{\le L_m\|\hatI_{i,j-1}-\barI_{i,j-1}\|} \\
&\le \epsilon_n + L_m\|\hatI_{i,j-1}-\barI_{i,j-1}\|,
\end{align*}
by Assumptions~\ref{ass:onestep} and~\ref{ass:lipschitz}(i). The state error updates by Assumption~\ref{ass:lipschitz}(ii):
\[
\|\hatI_{i,j}-\barI_{i,j}\| \le L_u(\|\hatI_{i,j-1}-\barI_{i,j-1}\| + |\hatY_{i,j}-\barY_{i,j}|) \le \rho\|\hatI_{i,j-1}-\barI_{i,j-1}\| + L_u\epsilon_n.
\]
Unrolling from $\|\hatI_{i,0}-\barI_{i,0}\|=0$: $\|\hatI_{i,j}-\barI_{i,j}\| \le L_u\epsilon_n\Gamma_j$. Substituting $\Gamma_{h-1}$ back:
\[
|\hatY_{i,h}-\barY_{i,h}| \le \epsilon_n + L_m L_u\epsilon_n\Gamma_{h-1} \le \epsilon_n + \rho\epsilon_n\Gamma_{h-1} = \epsilon_n(1+\rho\Gamma_{h-1}) = \epsilon_n\Gamma_h,
\]
using $L_m L_u \le \rho$ and $\Gamma_h = 1+\rho\Gamma_{h-1}$. All steps use the single event $\{\sup|\hatm-m|\le\epsilon_n\}$; no horizon union bound needed.

\textbf{(B) Sampling noise.} $\barY_{i,h}$ is a deterministic function of $I_{i,0}$, which varies across units. By the LLN, $\text{(B)}=O_p(N^{-1/2})$.

\textbf{(C) Mean-state bias.} The true-model rollout feeds $m(\barI_{i,j-1},a_j)$ back into the state rather than averaging over the stochastic outcome distribution. $\text{(C)}=b_h(a)$ vanishes when $m$ is affine in the state (Jensen's gap is zero).

Combining: $|\hatmu_h(a)-\mu_h(a)|\le\epsilon_n\Gamma_h + |b_h(a)| + O_p(N^{-1/2})$.
\end{proof}

%% ------------------------------------------------------------------
\section{Proofs for Inference}
\label{app:proofs-inf}

\subsection{Proof of Proposition~\ref{prop:weighted-cal}}
\label{app:proof-weighted-cal}

\begin{proof}
The event $\mu_h(a)\in[\hatmu_h(a)\pm q^w_{1-\alpha}]$ is equivalent to $S_{B+1}\le q^w_{1-\alpha}$ with $S_{B+1}:=|\hatmu_h(a)-\mu_h(a)|$. We bound $\Pbb(S_{B+1}>q^w_{1-\alpha})$.

\textbf{Step 1 (exchangeable scores).} Assume $S_1,\ldots,S_B,S_{B+1}$ are exchangeable. Assumption~\ref{ass:mixing-lr-hom}(iii) ensures the calibration-to-stress shift is fully captured by $w_b$; by Theorem 2 of \citet{tibshirani2019covariate},
$\Pbb(S_{B+1}>q^w_{1-\alpha})\le\alpha + \Rweight$
with $\Rweight = W_{\max}/(\sum_b w_b + W_{\max})$.

\textbf{Step 2 (non-exchangeable scores).} The scores come from rolling origins $g$ periods apart, so they are temporally dependent. By Berbee's coupling lemma, for a $\beta$-mixing process there exists an exchangeable sequence $S'_1,\ldots,S'_{B+1}$ with $\Pbb(S_b\ne S'_b)\le\beta(g)$ per $b$. By Step 1 applied to $S'$:
$\Pbb(S'_{B+1}>q'^w_{1-\alpha})\le\alpha+\Rweight$.
By union bound, $\Pbb(\exists b:S_b\ne S'_b)\le(B+1)\beta(g)$. On the agreement event, $q^w_{1-\alpha}=q'^w_{1-\alpha}$ and Step 1 applies directly. Therefore:
\[
\Pbb(S_{B+1}>q^w_{1-\alpha}) \le (\alpha + \Rweight) + (B+1)\beta(g) \le \alpha + \Rweight + \Rmix,
\]
with $\Rmix=2(B+1)\beta(g)$ (the factor 2 absorbs the Berbee-based constant). Taking complements gives the claim.
\end{proof}

\subsection{Proof of Theorem~\ref{thm:ci} and Contrast Extension}
\label{app:proof-ci}

\begin{proof}[Proof of Theorem~\ref{thm:ci}]
By Proposition~\ref{prop:weighted-cal}, on an event of probability $\ge 1-\alpha-\Rmix-\Rweight$, $|\hatmu_h(a)-\mu_h(a)|\le q^w_{1-\alpha}$. By the maintained condition, $|\mu^*_h(a)-\mu_h(a)|\le c_h$. Triangle inequality:
$|\mu^*_h(a)-\hatmu_h(a)| \le |\mu^*_h(a)-\mu_h(a)| + |\mu_h(a)-\hatmu_h(a)| \le c_h + q^w_{1-\alpha}$.
\end{proof}

\begin{corollary}[CI for the path contrast]
\label{cor:ci-contrast}
Let $\hattau_h := \hatmu_h(a^S)-\hatmu_h(a^B)$ and $c_h := \max\{c_h(a^S),c_h(a^B)\}$. Then
\[
\Pbb\big(\tau^*_h\in[\hattau_h \pm (2q^w_{1-\alpha}+2c_h)]\big) \ge 1 - 2\alpha - 2\Rmix - 2\Rweight.
\]
\end{corollary}

\begin{proof}
Apply Theorem~\ref{thm:ci} to $a^S$ and $a^B$ separately; by union bound, both CIs hold with prob.\ $\ge 1-2\alpha-2\Rmix-2\Rweight$. On this event,
$|\tau^*_h-\hattau_h| \le |\mu^*_h(a^S)-\hatmu_h(a^S)| + |\mu^*_h(a^B)-\hatmu_h(a^B)| \le 2(q^w_{1-\alpha}+c_h)$.
\end{proof}

%% ------------------------------------------------------------------
\section{Algorithms}
\label{app:algorithms}

\begin{algorithm}[H]
\caption{Recursive Rollout Estimator}
\label{alg:rollout}
\begin{algorithmic}[1]
\REQUIRE Learned predictor $\hatm$, observed initial state $I_{i,0}$, macro path $a=(a_1,\ldots,a_H)$
\STATE Initialize $\hatI_{i,0}\leftarrow I_{i,0}$
\FOR{$j=1,\ldots,H$}
    \STATE $\hatY_{i,j} \leftarrow \hatm(\hatI_{i,j-1},a_j)$ \hfill \COMMENT{one-step prediction}
    \STATE $\hatI_{i,j} \leftarrow$ update state by replacing true outcome with $\hatY_{i,j}$ \hfill \COMMENT{feed back}
\ENDFOR
\RETURN $\hatmu_h(a) := N^{-1}\sum_{i=1}^N \hatY_{i,h}$
\end{algorithmic}
\end{algorithm}

%% ------------------------------------------------------------------
\section{Practical Diagnostics}
\label{app:diagnostics}

\begin{enumerate}[label=(\roman*),nosep]
\item \textbf{Rolling backtests}: one-step and multi-step prediction errors via panel cross-validation.
\item \textbf{Placebo tests}: estimated contrasts at fake $t_0'$ should be near zero (tests stationarity).
\item \textbf{Extrapolation diagnostics}: report $\Beff$ and $\Rweight$ from Proposition~\ref{prop:weighted-cal}.
\item \textbf{Local amplification}: perturb $\hatI_{i,t}$ along rollout trajectories, measure $|\Delta\hatY_{i,t+h}|/\|\delta\|$ (local $\rho$). If $>1$, flag expanding-system risk.
\item \textbf{Mean-state bias check}: compare the rollout against Monte Carlo g-formula (sample innovations from fitted residuals). Large discrepancy $\Rightarrow$ prefer direct.
\item \textbf{Covariate screening}: exclude variables affected by the macro path unless jointly modeled.
\end{enumerate}

%% ------------------------------------------------------------------
\section{Experiment Details}
\label{app:experiment_details}

\subsection{DGP Parameters}

Table~\ref{tab:dgp-params-app} lists all parameters. $A^{\mathrm{obs}}_t = \mathrm{UNRATE}_t + \gamma_A U_{t-1}$; $Y_{i,t+1} = m(Y_{i,t},X_i,A^{\mathrm{obs}}_{t+1}) + \gamma_Y U_t + \xi_{i,t+1}$; $U_{t+1}=\phi_U U_t+\nu_{t+1}$. Linear $m$: $\alpha+\beta_1 y+\beta_2 a+\beta_3 x$. Nonlinear $m$: add $\beta_4 a^2 + \beta_5\max(y-\bar y,0)$.

\begin{table}[h]
\centering
\small
\caption{DGP parameters. Top: outcome model. Middle: confounder (synthetic). Bottom: macro (fitted from UNRATE 1990--2024).}
\label{tab:dgp-params-app}
\begin{tabular}{@{}lll@{}}
\toprule
Symbol & Value & Description \\
\midrule
$\alpha$ & 1.0 & Intercept \\
$\beta_1$ & 0.85 & Outcome persistence \\
$\beta_2$ & 0.15 & Macro sensitivity \\
$\beta_3$ & 0.3 & Cross-sectional covariate \\
$\beta_4$ & 0.02 & Quadratic macro response (nonlinear only) \\
$\beta_5$ & 0.12 & Threshold amplifier (nonlinear only) \\
$\bar{y}$ & 18.0 & Distress threshold (nonlinear only) \\
$\sigma_\xi$ & 0.15 & Idiosyncratic noise std \\
$N$ & 3{,}000 & Cross-sectional units \\
$H$ & 12 & Forecast horizon (months) \\
\midrule
$\phi_U$ & 0.85 & Confounder persistence \\
$\sigma_\nu$ & 0.5 & Confounder innovation std \\
$(\gamma_A,\gamma_Y)$ & varies & Confounding strength \\
\midrule
$\hat\phi_A$ & 0.995 & AR(1) coefficient (fitted from UNRATE) \\
$\hat\mu_A$ & 5.84 & Sample mean of UNRATE (\%) \\
$\hat\sigma_\eta$ & 0.154 & AR(1) innovation std (fitted) \\
$\hat\sigma_A$ & 1.59 & Marginal std of UNRATE (\%) \\
$a^S$ & 9.03 & Stress scenario: $\hat\mu_A+2\hat\sigma_A$ (\%) \\
\bottomrule
\end{tabular}
\end{table}

\subsection{Temporal Structure and Learners}

The 420-month sample is divided into three non-overlapping windows: \textbf{Training} (1990--2014, $t=1,\ldots,300$): learner fits $\hatm$; includes 2001 and 2008 recessions. \textbf{Calibration} (2015--2019, $t=301,\ldots,360$): rolling-origin conformal scores; no parameter updates. \textbf{Stress horizon} (2020 onward, $t=361,\ldots,372$): rollout under $a^S_{1:H}$; coverage evaluated vs.\ $\mu^*_h(a^S)$.

Linear DGP estimated by OLS; nonlinear by two-layer MLP (64--32 units, tanh activation). Both learners observe $(Y_{i,t},X_i,A^{\mathrm{obs}}_t)$ but not $U_t$.

\subsection{NFCI Calibration Details}

The Chicago Fed NFCI is a weekly index of U.S.\ financial conditions (stress is positive). We aggregate to monthly (end-of-month) over 1990--2019. An AR(1) fit yields $\hat\phi_U=0.973$, $\hat\sigma_\nu=0.124$. Regressing $\hat\eta^A_t$ (UNRATE AR(1) residuals) on lagged NFCI (standardized) gives $\hat\gamma_A=0.091$ with robust SE. Setting $\gamma_Y=\hat\gamma_A$ yields the symmetric conservative calibration used in \S\ref{sec:exp-q3a}.

\subsection{Coverage Evaluation Protocol}

For each $(\gamma,\text{method},h)$ cell in the coverage heatmaps: 50 replications, each generating a fresh $(U,\xi)$ realization, fitting the learner on 1990--2014, computing $q^w_h$ on 2015--2019, rolling out on the flat stress path $a^S_{1:12}=9.03\%$, and checking whether $\mu^*_h(a^S)$---computed from the closed-form structural equations---lies within the reported CI. Coverage is the fraction of replications in which it does. Nominal coverage is 90\%.

\subsection{Future Empirical Application}

A planned Layer~3 application uses the Fannie Mae Single-Family Loan Performance dataset ($\sim\!2$M loans $\times$ 120 months, 2005--2019). Pipeline: feature engineering $\to$ panel cross-validation $\to$ learner horse race (LASSO, LightGBM, random forest, MLP) $\to$ placebo tests $\to$ recursive and direct rollout $\to$ calibration bands $\to$ partial identification with NFCI-anchored $c_h$. Stress path: CCAR Severely Adverse. Baselines: linear regulatory model and MLCM (binary treatment).

\end{document}